\documentclass[a4paper,11pt,oneside]{article}
\usepackage[margin=1in]{geometry}
\usepackage{amsmath, amsthm, amssymb}
\usepackage{hyperref}
\usepackage{mathtools}

\hypersetup{
	colorlinks=true,
  linkcolor=red,          %
  citecolor=blue,         %
  filecolor=magenta,      %
  urlcolor=cyan           %
}

\usepackage[utf8]{inputenc} %
\usepackage[T1]{fontenc}    %
\usepackage{hyperref}       %
\usepackage{url}            %
\usepackage{booktabs}       %
\usepackage{amsfonts}       %
\usepackage{nicefrac}       %
\usepackage{microtype}      %

\usepackage{mathtools}
\usepackage{amsthm, amssymb}
\newtheorem{theorem}{Theorem}[section]
\newtheorem{lemma}{Lemma}[section]

\newtheorem{corollary}{Corollary}[section]
\newtheorem{rem}{Remark}[section]

\newtheorem{claim}{Claim}[section]
\newcommand{\bepf}{\begin{proof}}
\newcommand{\enpf}{\end{proof}}

\usepackage{stmaryrd}  %

\newcommand{\nrm}[1]{\left\Vert #1 \right\Vert}

\newcommand{\R}{\mathbb{R}}

\newcommand{\Z}{\mathbb{Z}}

\newcommand{\paren}[1]{\left( #1 \right)}
\newcommand{\sqprn}[1]{\left[ #1 \right]}

\newcommand{\set}[1]{\left\{ #1 \right\}}

\newcommand{\abs}[1]{\left| #1 \right|}
\newcommand{\iprod}[2]{\left\langle #1 , #2 \right\rangle}

\newcommand{\beq}{\begin{eqnarray*}}
\newcommand{\eeq}{\end{eqnarray*}}
\newcommand{\beqn}{\begin{eqnarray}}
\newcommand{\eeqn}{\end{eqnarray}}
\newcommand{\ben}{\begin{enumerate}}
\newcommand{\een}{\end{enumerate}}
\newcommand{\bit}{\begin{itemize}}
\newcommand{\eit}{\end{itemize}}
\newcommand{\hide}[1]{}
\newcommand{\oo}[1]{\frac{1}{#1}}
\newcommand{\eps}{\epsilon}

\newcommand{\argmin}{\mathop{\mathrm{argmin}}}

\newcommand{\diam}{\operatorname{diam}}

\newcommand{\dom}{\operatorname{dom}\,}
\newcommand{\Dom}{\operatorname{Dom}\,}

\newcommand{\Binom}{\mathrm{Bin}}

\newcommand{\calF}{\mathcal{F}}

\newcommand{\calN}{\mathcal{N}}

\newcommand{\calZ}{\mathcal{Z}}

\newcommand{\ds}{\displaystyle}

\newcommand{\X}{\mathcal{X}}

\renewcommand{\hbar}{{\bar h}}

\newcommand{\G}{\mathcal{G}}
\newcommand{\F}{\mathcal{F}}
\renewcommand{\H}{\mathcal{H}}
\def\eps{\varepsilon}

\newcommand{\E}{\mathbb{E}}

\renewcommand{\P}{\mathbb{P}}

\newcommand{\ddim}{\mathrm{ddim}}
\renewcommand{\diam}{\mathrm{diam}}

\renewcommand{\Binom}{\mathrm{Binomial}}

\usepackage{bbm}

\newcommand{\tsabs}[1]{| #1 |}
\newcommand{\wh}{\widehat}

\newcommand{\Lip}[1]{\nrm{#1}_{\textrm{{\tiny \textup{Lip}}}}}

\newcommand{\Liprho}[2]{\nrm{#1}_{\textrm{{\tiny \textup{Lip(\ensuremath{#2})}}}}}
\newcommand{\tsLip}[1]{\tsnrm{#1}_{\textrm{{\tiny \textup{Lip}}}}}
\newcommand{\tsnrm}[1]{\| #1 \|}

\newcommand{\citet}[1]{\cite{#1}}
\newcommand{\citep}[1]{\cite{#1}}

\newcommand{\rad}{\wh{\mathcal{R}}}

\newcommand{\risk}{R}

\newcommand{\logtwo}{\log_2}
\newcommand{\lognat}{\ln}

\newcommand{\LM}{\mathcal{L}}
\newcommand{\ldn}[1]{\|#1\|^*_w }
\newcommand{\lpn}[1]{\|#1\|_w }
\newcommand{\lgn}[1]{\lambda_f(#1)}
\newcommand{\lpnt}[1]{\|#1\|_{w,t}}
\newcommand{\ldnt}[1]{\|#1\|^*_{w,t}}
\newcommand{\mynorm}[2]{\| #1 \|_#2}
\newcommand{\infnrm}{\mynorm{\cdot}{\infty}}

\newcommand{\trunc}{\theta}
\newcommand{\truncb}{ \frac 1 \trunc}

\newcommand{\hypo}{h}

\renewenvironment{proof}{{\bf Proof:\ }}{\hfill$\Box$\medskip}

\newcommand{\dt}[1]{\tblue{\dt AG: #1}}

\usepackage[ruled,vlined]{algorithm2e}

\usepackage{import}

\usepackage{amsthm}
\usepackage{mathtools}
\usepackage{bbm}

\title{Non-parametric Binary regression in metric spaces with KL loss}
\author{
Ariel Avital,
Klim Efremenko,
Aryeh Kontorovich,
David Toplin,
Bo Waggoner
 \\\small{\texttt{
kish03@gmail.com,
klim@bgu.ac.il,
karyeh@cs.bgu.ac.il,
david.tolpin@gmail.com,
bwag@colorado.edu
}}}

\begin{document}
\maketitle

\begin{abstract}

We 
propose a non-parametric variant of binary regression, where the hypothesis is regularized to be a Lipschitz function taking a metric space to $[0,1]$ and the loss is logarithmic.
This setting presents novel computational and
statistical challenges.
On the computational front, we derive
a novel efficient optimization algorithm based on interior point methods; an attractive feature
is that it is parameter-free (i.e., does
not require tuning an update step size).
On the statistical front, the unbounded loss
function presents a problem for classic
generalization bounds, based on covering-number
and Rademacher techniques.
We get around this challenge via an adaptive
truncation approach, and also
present a lower bound indicating that
the truncation is, in some sense, necessary.


\end{abstract}

\section{Introduction}

The algorithmic and statistical aspects of
real-valued
non-parametric regression 
are largely understood \cite{tsybakov2008introduction,Gyrfi2002,RW05}.
At the opposite end of the spectrum is {\em binary} regression, where the dependent variable is $\set{0,1}$-valued. 
Binary regression 
has numerous applications: consider, for example, predicting the number of defects in a production batch or a number of visitors of a web site in a time unit. A commonly used 
parametric
latent variable model for binary regression
is {\em logistic regression}.
However, unlike
Gaussian process regression, there is no known closed form efficient solution for logistic regression, even in the simplest 
linear setting.
Non-parametric binary regression is a natural extension of linear logistic regression to cases where 
a non-linear
dependency between predictors and the probability of outcome of $1$ is desired.
Non-parametric binary regression models have
appeared in previous literature,
\citep{H83,ZH02,CGR07}, but,
to our knowledge, rigorous statistical and
computational aspects
have not been thoroughly addressed.

In this work, we propose a variant of non-parametric binary regression which allows both an efficient approximation algorithm and a rigorous theoretical analysis.  We consider the following problem setting: 

\paragraph{Problem setting.}{
Let $(\X,\rho)$ be a metric space endowed with a distribution $\mu$, 
and let $\hypo$ be a mapping $\hypo : \X \to [0,1]$. The learner observes $n$ iid draws $\paren{X_i,Y_i}$, where 
$(X_i,Y_i) \sim \mu$
Our goal is to estimate $\hypo$ point-wise using the sample, where the loss is defined in terms of Kullback-Liebler (KL) divergence:
\beqn
\label{eq:lossdef}
\ell(y,\hypo(x)) = -y \ln(\hypo(x)) - (1-y) \ln(1-\hypo(x)).
\eeqn
To make the problem well-posed (and to regularize against overfitting),
we impose an $L$-Lipschitz condition on $\hypo$
with respect to the metric $\rho$.
This suggests a natural optimization problem: minimize the empirical risk under a smoothness
constraint.
Having chosen an optimal $\hypo$ on the labeled sample, we then Lipschitz-extend it
to the whole space using standard techniques.
As a technicality, we adaptively truncate $\hypo$ to keep it bounded away from $0$ and $1$;
this enables fully empirical finite-sample guarantees, and also turns out to be, in some sense,
necessary.}

\paragraph{Our contributions.}{
Aside from the conceptual problem setting, we provide several statistical and algorithmic results.
\begin{itemize}
\item an efficient algorithm (Algorithm~\ref{llalgo}) for solving the optimization problem implied by our learning setting:
computing an $L$-Lipschitz $\hypo$ that minimizes the empirical risk with respect to the loss (\ref{eq:lossdef}).
\item a generalization bound (Theorem~\ref{thm:main-risk}) based on covering numbers and an adaptive truncation
\item a lower bound (Theorem~\ref{thm:lb-realizble},~\ref{thm:lb-agonstic}) indication that no non-trivial generalization bound is possible without some truncation.
\end{itemize}
}

\section{Related Work}


Non-parametric regression is well studied in a general setting~\cite{Simonoff1996,Gyrfi2002,Wasserman2006,tsybakov2009}.
Non-parametric binary regression has been employed in a number of works.
\citet{H83} gives a statistical recipe
for binary regression using local logistic regression on spans over the
predictor variable; references therein
point to
earlier works on non-linear and
non-parametric binary regression.

A well-studied approach to non-parametric binary regression is Kernel
logistic regression (KLR). \citet{ZH02} provides some results on KLR as well as literature
overview. A different but related approach to non-parametric binary
regression involves Gaussian process prior on the response probability
function. \citet{CGR07} gives a practical description of the
method as long as overview of related prior research.

Interior point method (IPM) has vast literature. \cite{Lesaja} gives a brief historical review of the development of IPM, wheres a short survey on the different variants can be found in \cite{Glavic}. We've decided to stick to IPM as presented in \cite{nesterov}, as it gives great detailed technical overview and complexity bounds, while avoiding inner loops and line-searching. The first publication we've found to describe usage of IPM to solve Maximum-Likelihood appears in \cite{Terlaky}, discussing usage of both barrier method and primal-dual IPM. \cite{Koenker} compares between EM and IPM schemes in the context of non-parametric maximum likelihood, where the latter performs better on the dual problem. \cite{kim} suggests a SQP approach for mixture proportions instead of IPM.


\section{Technical Background}

\label{sec:tech}

We write $\lognat$ for the natural logarithm
and $\log_b$ to specify a different base $b$.

\paragraph*{Metric spaces, Lipschitz constants.}
A {\em metric} $\rho$ on a set $\X$ is a symmetric function 
that is positive (except for $\rho(x,x)=0$) 
and satisfies the triangle inequality $\rho(x,y) \le \rho(x,z)+\rho(z,y)$;
together the two comprise the metric space $(\X,\rho)$.
The diameter of a set $A\subseteq\X$
is defined by $\diam(A) = \sup_{x,y\in A} \rho(x, y)$.
There is no loss of generality in assuming $\diam(\X)=1$ 
since we can always scale the distances (when they are bounded).
The {\em Lipschitz constant} of a function $f :\X\to\R$, denoted $\Lip{f}$
(or $\Liprho{f}{\rho}$ if we wish to make the metric explicit)
is defined to be the smallest $L \ge 0$ such that
$|f(x)-f(y)| \le L\rho(x, y)$ 
holds for all $x, y \in\X$. 
In addition to the metric $\rho$ on $\X$,
we will endow the space of all functions $f:\X\to\R$
with
the $L_\infty$ metric:
\beq
\nrm{f-g}_\infty =
\sup_{x\in\X}\abs{f(x)-g(x)}
.
\eeq
A function is called \emph{$L$-Lipschitz} if $\Lip{f}\leq L$.
We will denote by $\H_L$ the collection of all $L$-Lipschitz functions $\X\to[0,1]$.
It will occasionally be convenient to restrict this class to functions
with $\Lip{f}\ge1$; the latter collection will be denoted by $\H_{L\ge1}$.
This incurs no loss of generality in our results, as 
our Structural Risk Minimization procedure in general
selects hypotheses whose Lipschitz constant grows with sample size.
(See for example the risk bound presented at the beginning of Section \ref{sec:upperbound}.)

\paragraph*{Doubling dimension.}
For a metric space $(\X,\rho)$, let
$\lambda>0$
be the smallest value such that every
ball in $\X$ can be covered by $\lambda$ balls of half the radius.
The {\em doubling dimension} of $\X$ is $\ddim(\X)=\logtwo \lambda$.
A metric space (or family of metrics) is called {\em doubling}
if its doubling dimension is uniformly bounded. 

Doubling metric spaces occur naturally in many data analysis applications,
including for instance the geodesic distance of a low-dimensional manifold 
residing in a possibly high-dimensional space 
assuming mild conditions, e.g., on curvature.
Some concrete examples for doubling metric spaces include:
(i) $\R^d$ for fixed $d$ equipped with an arbitrary
norm
(ii) the planar earthmover metric between point sets of fixed size $k$
\cite{gottlieb2014efficient};
(iii) the $n$-cycle graph and its continuous version, the quotient $\R/\Z$,
and similarly bounded-dimensional tori.

\paragraph{Self-concordance functions and barriers.}
\label{sc-tech}
A function $f$ is called self-concordant (s.c.) if there exists a constant $M_f \ge 0$ s.t. the following holds for all $w \in \dom f$ and $u,v\in \R^n$ :

$$
\iprod{f'''(w)[u]v}{v}
\le
2 M_f \| u \|_w^3,
$$
where
$$ 
f'''(w) = \bigg|
\lim_{\alpha \rightarrow 0}{\frac 1\alpha \bigg[
\nabla^2 f(w + \alpha u) - \nabla^2 f(w)
\bigg]}
\bigg|,
$$
and $\|u\|_w$ is the primal local norm of $f$, defined as
$$ \lpn{u} = {\iprod{\nabla^2f(w)u}{u}}^{1/2}. $$
In addition, we define the dual local norm as
$$ \ldn{u} = {\iprod{[\nabla^2f(w)]^{-1}u}{u}}^{1/2}. $$
When the latter is applied to the gradient of $f$,
we have the so-called {\em local norm of the gradient}:
\begin{align*}
& \lambda_f(w) =  \ldn{\nabla f(w)} = \\
& {\iprod{[\nabla^2f(w)]^{-1} \nabla f(w)}{\nabla f(w)}}^{1/2}.
\end{align*}
We say that
$f$ is a {\em standard} s.c. function if $M_f=1$.
Let $F$ be a standard s.c. function. We call it a s.c. barrier for the set $\Dom F$, 
if for all $w\in \dom F$ we have:
$$
\sup_{u\in \R^n}{
\Bigg[
2 \iprod{\nabla F(w)}{u}
-
\iprod{\nabla^2 F(w)u}{u}
\Bigg]
\le v.
}
$$
Note the difference between the domain of $F$, $\dom F$, and the set for it is considered to be a barrier, $Dom F$.
For non-degenerate $F$, the left-hand side could be replaced by ${\lambda^2_F(w)}$.
For equivalent definitions and more, see \cite[Chapter 5]{nesterov}.

\section{Regression Algorithm}

\label{sec:algo}

Our learning setting entails solving the following optimization problem:
given the sample $(X_i,Y_i)_{i\in[n]}$, 
we wish to compute %
$w\in[0,1]^n$
that minimizes the empirical risk
\beq
R_n(w):=
\sum_{i=1}^n\sqprn{
-Y_i\ln(w_i)-(1-Y_i)\ln(1-w_i)
},
\eeq
subject to the truncation constraints
$\trunc\le w_i\le1-\trunc$ (the value of $\trunc$ will be determined in the sequel, see subsection  ~\ref{subsec:trunc}; for
now it is a fixed parameter),
and the Lipschitz constraints
$|w_i-w_j|\le L\rho(X_i,X_j)$, for all $i,j\in[n]$.

We observe right away that our objective function is strictly convex
and the feasible set is linearly constrained, hence the problem
has a unique minimizer $w^*$.
The main result of this section is
\begin{theorem}
\label{thm:llalgo_thm}
An $\eps$-additive
approximation to the problem stated 
above can be computed in time  $O\paren{n^4 \ln \frac n{\eps}}$.
\end{theorem}

Our algorithm relies on interior point methods, and is a close variant of the one
presented in
\cite{nesterov}.
Throughout this section,
we denote by 
the feasible set by
$Q$ 
the objective function by $f_0$,
and
the barrier function for the set $Q$
by
$F$.
We assume both $f_0,F$ to be non-degenerate (in the sense that the Hessian is positive-definite) self-concordant functions, 
and in particular,
$F$
is a $v$-self-concordant barrier.

$\beta,\gamma$ are constants to be specified later in this section.
In our case, the objective is $R_n(w)$, and $F$ enforces the solution to be $L$-Lipschitz inside $[\trunc,1-\trunc]^n$.

\begin{algorithm}
\label{llalgo}
\SetAlgoLined
 initialization: $t=0,k=1$\;
 \While {$t_k < \frac{\beta(\beta+\sqrt{v})}{\eps (1-\beta) } $}{
  $t_k = t_{k-1} + \frac{\gamma}{\|\nabla f_0(w_{k-1})\|_{w_{k-1},t_{k-1}}^*}$\;
  $w_k = w_{k-1} + [t_k \nabla^2 f_0(w) + \nabla^2 F(w)]^{-1} (t_k \nabla f_0(w) + \nabla F(w))$\;
  $ k = k + 1$ \;
 }
\caption{Path-following for Log-Likelihood}
\end{algorithm}

\subsection{Interior point method and Logarithmic functions}{
In \cite{nesterov}, the author defines the path-following scheme for linear functions, then generalizes it to non-linear functions by adding the objective as a constraint, thus minimizing the epigraph of the objective. Although we could utilize this technique,\footnote{A barrier for epigraph of $-\log(x)$  can be obtained, see \cite[Theorem 5.3.5]{nesterov}.} it requires adjustments to the LP framework, and additional auxiliary path-following iterations. Instead, we opt for minimizing the objective directly
---
an approach applicable to any self-concordant function. 
Additionally, this allows acceleration of the path-following scheme, 
as the step size is tied to the dual norm of a non-linear gradient. 

}

\paragraph{Consistency.}{

The path-following scheme is a succession of one-step Newton method iterations, where we gradually give more and more weight to the objective. In 
order to
guarantee quadratic convergence, 
the condition $\lgn{w}<\frac 1 {M_f}$ is required throughout the whole process: After every Newton step and increment in $t$, the path-following parameter.
Let denote our objective $f_0(w)$, $F(w)$ as the barrier. Then the path-following objective becomes
\begin{gather*}
f(w;t)=tf_0(w) + F(w),
\end{gather*}
with self-concordant parameter, $M_f(t) = \max{\{M_{f_0} / \sqrt t , M_F\}}$.
In addition, recall the definitions of the primal and dual local norms (See technical background).
We now extend these notations to include $t$:
\beq
\lpnt{u} &=& {\iprod{\nabla^2f(w;t) u}{u}}^{1/2}\\
\ldnt{u} &=& {\iprod{[\nabla^2f(w;t)]^{-1}. u}{u}}^{1/2}
\eeq
Further applying it to the definition of the local norm of the gradient,
\begin{align*}
&\lambda_{w,t} = 
\ldnt{ \nabla f(w,t) } = \\
&\langle [\nabla^2f(w;t)]^{-1} \nabla f(w;t),
\nabla f(w;t) 
\rangle^{1/2}.
\end{align*}

Note that the derivatives are taken only with respect to $w$.
Before we proceed to prove consistency of path-following for non-linear objective, we invoke a
standard fact in positive semi-definite order to prove a lemma.
Note that since $f_0,F$ are non-degenerate, then so is $f(w;t)$. Thus, $\nabla^2f(w;t)>0$ for all $w\in \dom f$.

The following is a standard fact \cite[Corollary 7.7.4(a)]{horn}:
\begin{lemma}
\label{sdo}
Let $A,B $ be $n \times n$ symmetric matrices. If $0 \preceq A \preceq B$ then $A^{-1} \succeq B^{-1}$.
\end{lemma}
Now let us apply it to the local norm.

\begin{lemma}
\label{gno}
Let $\nabla^2f(w;t) > 0 $, then for all $t \ge 0$, $\Delta t > 0$ we have
\beq \|u\|_{w,t_+ \Delta t}^* \le \|u\|_{w,t}^*
.
\eeq
\end{lemma}
\bepf
Note that $$ \nabla^2f(w;t + \Delta t) \succeq \nabla^2f(w;t). $$
Using Lemma \ref{sdo}:
\begin{align*}
&\|u\|_{w,t_+ \Delta t}^* = 
{\langle [\nabla^2 f(w;t + \Delta t)]^{-1} u , u \rangle}^{1/2} \le \\
&{\langle [\nabla^2 f(w;t)]^{-1} u , u \rangle}^{1/2} = 
\|u\|_{w,t}^*.
\end{align*}
\enpf

The following argument closely follows \cite[Lemma 5.2.2]{nesterov}.
\begin{theorem}
Let $t_+=t + \frac{\gamma}{M_f(t)\|\nabla f(w;t)\|_{w,t}^*} $ and let the pair $(w,t)$ satisfy:
\begin{gather*}
\|\nabla f(w;t)\|_{w,t}^* \le \frac {\beta}{M_f(t)}\\
where \,\,  \beta =\frac{\tau^2}{\paren{1-\tau}^2},\quad \tau \le \frac 12 .
\end{gather*}

Then for $\gamma$ satisfying
\beq |\gamma| \le \tau - \frac{\tau^2}{\paren{1-\tau}^2}, \eeq
we have again $\|\nabla f(w_+;t_+)\|_{w_+,t_+}^* \le \frac {\beta}{M_f(t_+)} $.
\end{theorem}

\bepf
First we apply Lemma \ref{gno} to switch between the norms:
\begin{gather*}
\|\nabla f(w,t_+)\|_{w,t_+}^* \le 
\|\nabla f(w,t_+)\|_{w,t}^*.
\end{gather*}
Then
\begin{align*}
&\|\nabla f(w,t_+)\|_{w,t}^* \le \\
&\|t \nabla f(w) + \nabla F(w) \|_{w,t}^* + \Delta t \|\nabla f(w) \|_{w,t}^* \le \\
&\frac {\beta + \gamma}{M_f(t)} = 
\frac {\tau}{M_f(t)} \le \frac{\tau}{M_f(t_+)}.
\end{align*}
Using \cite[Theorem 5.2.2]{nesterov} (i) we get:
\begin{align*}
&\|\nabla f(w_+;t_+)\|_{w_+,t_+}^* \le \\
&\frac{M_f(t_+)\paren{\|\nabla f(w,t_+)\|_{w,t_+}^*}^2}{\paren{1-M_f(t_+)\|\nabla f(w,t_+)\|_{w,t_+}^*}^2} \le \\
&\frac{\frac{\tau}{M_f(t_+)}}{\paren{1-\tau}^2}=
\frac{\tau}{M_f(\tau_+)} .
\end{align*}
\enpf

\begin{rem} We can assume $M_f(t) = 1$ from now on; the proof of the previous result
shows that this incurs no loss of generality.
\end{rem}
}

\subsection{Runtime complexity}{
In order 
obtain runtime guarantees,
we need to show that both the functional gap and local norm of objective gradient decrease with $t$.
These theorems are parallel to \cite[Theorem 5.3.10, Lemma 5.3.2]{nesterov}, respectively. Their proof can be found in the appendix.

\begin{theorem}
\label{thm:func-gap}
Let point $w_t$ satisfy $\|\nabla f(w_t;t)\|_{w_t,t}^* \le \beta $.
Then
$$ f_0(w_t)-f_0(w_{opt}) \le \frac{1}{t}\bigg(v+ \frac{\beta(\beta+\sqrt{v})}{1-\beta} \bigg),$$
where
$$w_{opt} = \argmin_{w \in \dom F} f(w). $$
\end{theorem}

\begin{lemma}
\label{lem:inc_in_t}
Let $\|\nabla f(w;t)\|_{w,t}^* \le \beta $, then
$$ \|\nabla f_0(w)\|_{w,t}^* \le \frac{1}{t}(\beta + \sqrt{v} + \gamma) . $$
\end{lemma}

We now proceed to the analytic complexity bound, which closely follows \cite[Theorem 5.3.11]{nesterov}. The proof appears in the appendix.
\begin{theorem}
\label{thm:algo-analytic-comp}
The maximum number of iterations of the above scheme is
$$O\paren{\sqrt{v} \ln{\frac{v \|\nabla f(w_0)\|^*_{w_0,0}}{\eps}} }.$$

\end{theorem}

}

\subsection{Runtime of Log-likelihood optimizer}{
The above states the number of path-following iterations. 
Now we'll derive complete runtime analysis for our problem.
First, Let us choose $\tau$ so as to maximize our step in $t$:
$$ \tau = 0.2291 ,\quad \gamma = 0.14,\quad \beta = 0.088. $$

Secondly, we construct a barrier for the domain of our problem.
From this construction we obtain the barrier parameter, $v= (n-1)n$.

A quick observation reveals that the analytic center is trivial. The Lipschitz barrier is minimized when all coordinates are equal. Further setting them all to $\frac 12$ also minimizes the box constraints. Thus, we can skip auxiliary path-following, setting $w_0=\frac 12 \cdot 1_n$.

Additionally, it can be shown that:
$$ \|\nabla f(w_0)\|^*_{w_0 , t_0} \le \sqrt n. $$
A detailed derivation of these
is presented in the Supplementary Material.

Finally, we'll need to derive the arithmetical complexity of each iteration.

Each iteration involves two complexities: One is of the oracle, i.e. calculation of gradient and Hessian at current point, the other is calculation of the newton step.
The derivatives calculation is dominated by the Lipschitz barrier gradient and hessian, where each requires $O(n^2)$ operations.
Calculating the next step in $w$ (and in t) is dominated by the computation the inverse of $\nabla ^2 f(w;t)$. As this is a $n \times n$ matrix, this involves $O(n^3)$ operations.
Thus, the arithmetical complexity of a single path-following iteration is $O(n^3)$, for a total of $O(n^4 \ln {\frac n \eps })$.

}

\subsection{Lipschitz extension}{
\label{sec:lipext}

In this section, we show how to evaluate our hypothesis on a new point.
Having computed an optimal hypothesis 
$\tilde\hypo$
on the sample $S$ via Algorithm~\ref{llalgo},
we wish to 
evaluate its prediction at a test point
$x \notin S$
via the Lipschitz extension technique.
As this method is by now standard
\citep{DBLP:journals/jmlr/LuxburgB04, GottliebKK13-simbad+IEEE, gottlieb2014efficient},
we present a brief sketch only.
Formally, if $S=\set{x_1,\ldots,x_n}$
and $\tilde\hypo :S\to \R$,
we wish to compute a value $y = \tilde\hypo(x)$
that minimizes
${\ds \max_{i\in[n]}  \frac{|y-\tilde\hypo(x_i)|}{\rho(x,x_i)}  }$.
By the McShane-Whitney extension theorem \citep{mcshane1934,Whitney1934}, the extension of $\tilde\hypo$ to the new 
point does not increase the Lipschitz
constant of $\tilde\hypo$, and so the risk bound in Theorem~\ref{thm:main-risk} 
applies.

The exact Lipschitz extension label $y$ of $x \notin S$ will always be determined
by a  pair of points $x_i,x_j\in S$, one with label greater than $y$ and one with a label less than $y$,
s.t. :
$$L
= \frac{\tilde\hypo(x_i)-y}{\rho(x,x_i)}
= \frac{y-\tilde\hypo(x_j)}{\rho(x,x_j)}
\ge \frac{|y-\tilde\hypo(x')|}{\rho(x,x')},x'\in S
$$
Note that $y$ cannot be increased or decreased without increasing the Lipschitz 
constant with respect to one of these points. 
Therefore, an exact Lipschitz extension may be computed in $O(n^2)$ time in brute-force fashion, 
by enumerating all point pairs in $S$, calculating the exact 
Lipschitz extension for $x$ with respect to each pair alone, i.e.,
$$ y_{ij} = \frac{h(x_j)\rho(x,x_i)+h(x_i)\rho(x,x_j)}{\rho(x,x_i)+\rho(x,x_j)} $$
and then select the one that achieves the highest Lipschitz constant.
}

\section{Lower bound}

\label{sec:lb}

The next theorem shows the necessity of truncation,
at least when selecting hypotheses via Empirical Risk Minimization (ERM).
We do this by demonstrating that minimizing empirical risk over untruncated classes (with predictions arbitrarily close to $0$ or $1$) leads to arbitrarily bad excess risk.
Recall our loss function $\ell$ defined in (\ref{eq:lossdef}).
For a fixed distribution $\mu$ over $\X\times\set{0,1}$,
the expected risk of a hypothesis $\hypo:\X\to[0,1]$
is $R(h)=\E_{(X,Y)\sim\mu}[\ell(Y,h(X))]$.
The minimizer\footnote{Assuming a unique minimizer
incurs no loss of generality.} of the empirical risk
is
called Bayes-optimal
and denoted by $h^*$. If $(X_i,Y_i)_{i\in[n]}$ is drawn from $\mu^n$,
it induces the empirical measure $\mu_n$ and empirical risk
$R_n(h)=\E_{(X,Y)\sim\mu_n}[\ell(Y,h(X))]$;
the empirical risk minimizer (whose existence and uniqueness
were surmised above from convexity) is denoted by $\hat h_n$.
A given hypothesis class $\H\subseteq[0,1]^\X$
induces the {\em realizable}
setting if
$h^*\in\H$; otherwise, the setting is agnostic.
Finally, we say that $\H$ is $\trunc$-truncated
if $\H\subseteq[\trunc,1-\trunc]^\X$.

\emph{Observation on risks:} Because entropy of a Bernoulli
random variable
is bounded, the risk of the Bayes optimal is bounded
by $\ln 2$.
However, the loss on any data point $(x,y)$ is not bounded,
and the risk of an arbitrary hypothesis is not bounded.

\subsection{Necessity of truncation for the realizable case}
The key point of the following example is that if we only observe samples of the form $(x,0)$ but never $(x,1)$, then empirical risk is minimized by choosing $h(x)$ as small as possible; but this could be a ruinous choice for expected risk. Truncation is a means of guarding against this.

\begin{theorem} \label{thm:realizable-lower}
\label{thm:lb-realizble}
  Suppose a hypothesis class $\H$ is not $e^{-O(n)}$-truncated.
  Then for any $\eps > 0$ and large enough $n$, there exists a distribution
  where $R(\hat h_n) \ge R(h^*) + \eps$
  with probability at least $\frac{1}{2}$.
\end{theorem}
\begin{proof}
  Given $\eps > 0$ and large enough $n$, choose $h,x$ such that (without loss of generality) $h(x) \leq e^{-4\eps n}$.
  This is possible because $\H$ is not $e^{-O(n)}$-truncated.
  Consider a distribution on $\X = \{x\}$ with $h^*(x) = \frac{1}{2n}$.

  By a union bound, with probability at least $\frac{1}{2}$,
  every example in the sample is of the form $(x,0)$,
  and none is of the form $(x,1)$.
  In this event, empirical risk of any hypothesis $h$
  depends only on $h(x)$ and is monotonically decreasing in $h(x)$.
  So in this event, we have $\hat h_n(x) \leq e^{-4\eps n}$,
  because in the worst case we have $\hat h_n = h$.
  But in this event, the excess risk of $\hat h_n$
  is minimized when $\hat h_n(x)$ is closest to
  $h^*(x) = \frac{1}{2n}$.
  So in this event, letting $H(p)$ be the binary entropy function, we have
  \begin{align*}
      &= R(\hat h_n)- R(h^*)\\
      &\geq \left[ \frac{1}{2n} \ln e^{4\eps n} + (1 - \frac{1}{2n}) \ln \frac{1}{1 - e^{4\eps n}}  \right] - H(\frac{1}{2n})  \\
      &\geq 2\eps - H(\frac{1}{2n})  \\
      &\geq \eps \hspace{20ex}\text{(for large enough $n$)}.  \qedhere
  \end{align*}
\end{proof}
\\
We note that in particular if $\H$ is not $e^{-n}$ truncated, then we can take $\eps = 1$ in the above proof and obtain an impossibility result for all sample sizes $n \geq 1$.

\subsection{Necessity of truncation for the agnostic case}

Now we argue that agnostic learning is unlikely to be achievable without
$e^{-O(\sqrt{n})}$-truncating the hypothesis classes.
We use a simple family of examples where, with constant probability, ERM selects a hypothesis whose risk is worse than the optimal-in-class by a constant.
\begin{theorem}
\label{thm:lb-agonstic}
  There exist constants $\eps,\delta > 0$ and input distribution $\mu$ such that, for any number of samples $n$, the following holds: there is a hypothesis class $\H$ of size two that is not $e^{-\sqrt{n}}$ truncated such that $\Pr[R(\hat{h}_n) \geq \min_{h \in \H} R(h) + \eps] \geq \delta$.
\end{theorem}
\begin{proof}
  Below, we will define the instance and family $\H$ of the form $\{h_1,h_2\}$.
  We then show in Claim \ref{claim:agnostic-gap} that, for all $n$, $\H$ has $R(h_2) \geq R(h_1) + \Omega(1)$.
  Finally, Claim \ref{claim:agnostic-prob} will show that $\Pr[\hat{h}_n = h_2] \geq \Omega(1)$.
\end{proof}

It will be convenient here to take the base of the logarithm in
(\ref{eq:lossdef}) to be $2$ rather than $e$.
We consider a three-point space,
$\X = \{1,2,3\}$.
The distribution $\mu$ on $\X \times \{0,1\}$ is completely uniform.
In particular, $h^*(x) = 0.5$ for all $x \in \X$.

Now define the hypothesis class $\H$, containing two hypotheses (we will pick $C$ based on $n$ later):

\begin{tabular}{llll}
        & $x=1$  & $x=2$  & $x=3$  \\
[1ex]
$h_1(x)$ & $\frac{1}{2}$ & $\frac{1}{2}$ & $2^{-C}$  \\
[1ex]
$h_2(x)$ & $\frac{1}{4}$ & $2^{-C}$ & $\frac{1}{2}$
\end{tabular}

\vskip1em
\begin{claim} \label{claim:agnostic-gap}
  $R(h_2) - R(h_1) \geq 0.04$ for any $C > 0$.
\end{claim}
\begin{proof}
  The hypotheses are symmetric conditioned on $x \neq 1$.
  Conditioned on $x=1$, which case occurs with probability $\frac{1}{3}$, the difference in expected loss is $\left[\frac{1}{2}\log(4) + \frac{1}{2}\log(\tfrac{4}{3})\right] - \left[\frac{1}{2}\log(2) + \frac{1}{2}\log(2)\right] = \frac{1}{2}\log(\tfrac{4}{3})$.
  The difference in risks is $\frac{1}{6}\log(\tfrac{4}{3}) = 0.0479\dots$.
\end{proof}

\begin{claim} \label{claim:agnostic-prob}
  With $n$ examples, suppose $C > \sqrt{n}$, implying that $\H$ is not
  $2^{-\sqrt{n}}$-truncated.
Then,
with constant probability,
the empirical risk of $h_2$ is smaller than that of $h_1$.
\end{claim}

First we give a sketch, then the proof.

\begin{proof}[Sketch]
  First, the difference in empirical risks due to samples where $X=1$ is $O(n)$.
  Second, the difference due to samples where $X=2$ or $X=3$ is $\pm \Omega(n)$ with constant probability.
  So in total, with constant probability $h_1$ has larger empirical loss.
  
  To see why there can be such a large difference in empirical loss due to samples where $X=2$ or $X=3$: when we obtain a sample where $X=2$ or $X=3$, with probability $\frac{1}{2}$ one of the hypotheses suffers a loss of $\log 2^C = C = \Omega(\sqrt{n})$.
  With constant probability, this occurs $\Omega(\sqrt{n})$ times more for $h_2$ than for $h_1$.
  In this case $h_2$ suffers an empirical loss on these samples that is larger by $\Omega(\sqrt{n}) \Omega(\sqrt{n}) = \Omega(n)$.
\end{proof}

\begin{proof}[Full]
  Write $\hat{R}^x(h_i)$ for the empirical loss of hypothesis $h_i$ on samples $(X,Y)$ where $X=x$.
  That is, $\hat{R}^x(h_i) = \sum_{j : x_j = x} L(h_i(x_j), y_j)$.
  We note $n \cdot R_n(h_i) = \hat{R}^1(h_i) + \hat{R}^2(h_i) + \hat{R}^3(h_i)$.
  
  Now, let $N(j)$ be the number of examples where $X=j$ and $N(j,k)$ be the
  number where $X=j$ and $Y=k$.
  We calculate:
  \begin{align*}
    \hat{R}^1(h_1) &= N(1),  \\
    \hat{R}^1(h_2) &= N(1,0) \log \frac{4}{3} + 2 N(1,1)  \\
                   &\leq 2N(1),  \\
    \hat{R}^2(h_1) &= N(2),  \\
    \hat{R}^2(h_2) &= N(2,0), \log\frac{1}{1-2^{-C}} + N(2,1) C,  \\
    \hat{R}^3(h_1) &= N(3,0), \log\frac{1}{1-2^{-C}} + N(3,1) C,  \\
    \hat{R}^3(h_2) &= N(3).
  \end{align*}
  We compute $R_n(h_1) - \R_n(h_2)$ by dividing the terms into two parts.
  First,
  \begin{align*}
    &\hat{R}^1(h_1) + \hat{R}^2(h_1) - \hat{R}^1(h_2) - \hat{R}^3(h_2)  \\
      &\geq N(1) + N(2) - 2N(1) - N(3)  \\
      &=    N(2) - N(1) - N(3) 
      \geq - n .
  \end{align*}
  Second, let $\beta = \log \frac{1}{1-2^{-C}}$.
  Because $C \geq 1$, $\beta \leq 1$.
  \begin{align*}
    & \hat{R}^3(h_1) - \hat{R}^2(h_2) \\
    &= \beta \left(N(3,0) - N(2,0)\right) + C \left(N(3,1) - N(2,1)\right)  \\
    &\geq -n + C\left(N(3,1) - N(2,1)\right) .
  \end{align*}
  In total, we get
    \[ n \cdot R_n(h_1) - n \cdot R_n(h_2) \geq -2n + C \left(N(3,1) - N(2,1)\right) . \]
It follows from Corollary~\ref{cor:binom-diff}
that with constant probability, $N(3,1) - N(2,1) > 2\sqrt{n}$.
  So with constant probability, $n \cdot R_n(h_1) - n \cdot R_n(h_2) > -2n + \sqrt{n} (2\sqrt{n}) = 0$.
\end{proof}

\subsection{Almost-matching upper bound for finite hypothesis classes}

We observe that the above lower bound is tight, in a sense: We can agnostically learn, given a finite hypothesis class $\H$, if it is $e^{-o(\sqrt{n})}$-truncated.
This may be surprising, as with only $n$ samples, it is in a sense not possible even to tell whether $h^*$ is only $\frac{1}{\Omega(n)}$ truncated (see Theorem \ref{thm:realizable-lower}).
We sketch the idea here, again using base-$2$ loss.

Suppose $H$ is $2^{-\trunc}$-truncated.
The loss on any given data point is between zero and $\log 2^{\trunc} = \trunc$.
So the empirical loss $R_n$ of a hypothesis, by Hoeffding's inequality, satisfies
  \[ \P\paren{\abs{R_n(h) - R(h)} \geq \eps } \leq 2e^{-2 n \eps^2 / \trunc^2} . \]
Union-bounding over the finite hypothesis class, we obtain that with $n$ samples, except with probability $\delta$, we can obtain the hypothesis of optimal risk up to error at most
  \[ \eps ~ = ~ \trunc \sqrt{\frac{\log \frac{2|\H|}{\delta}}{2n}} . \]
For $\trunc = o(\sqrt{n/\log|\H|})$, this goes to zero by taking more and more samples $n$.


\section{Upper bound}

\label{sec:upperbound}
The algorithm in Section \ref{sec:algo} produces 
a hypothesis $h
\in \H_L$.
This section is devoted to proving that with high probability,
$\risk(h,q)$ is not much greater than $\risk_n(h,q)$.

\begin{theorem}
\label{thm:main-risk}

For $L\ge1$ and every $\delta>0$, we have that,
with probability at least $1-\delta$,
\beqn
\risk(h) - \risk_n(h) \le  
O\paren{ \truncb \paren{\frac {L^d} n}^\frac 1 {d+1}}
+3 \ln \truncb \sqrt{\frac{\lognat(2/\delta)}{2n}}
\eeqn
holds uniformly over all $L$-Lipschitz $[\trunc,1-\trunc]$-valued hypothesis,
where $Z=(X_i,Y_i)_{i\in[n]}$ is the training sample.
\end{theorem}
\paragraph{Remark.} Throughout
the paper, we are treating
the Lipschitz constant $L$ of the hypothesis
as known in advance. In practice, it would
be chosen by cross-validation or Structural
Risk Minimization \citep{DBLP:journals/tit/Shawe-TaylorBWA98}.
Both techniques are standard,
and we defer their detailed application
to our case to the journal version.

\subsection{Covering numbers for Lipschitz function classes}
\label{sec:cov-num}
We begin by
obtaining complexity estimates
for Lipschitz functions in doubling spaces. 
We obtain simple and tight
bounds by direct control over the covering numbers.

The following covering number lemma appears in
\citet{GottliebKK13-simbad+IEEE}; we state it here with slightly better constants:
\begin{lemma}
\label{lem:cov-cov}
Let $\F_L$ be the collection of $L$-Lipschitz 
functions 
mapping
the metric space $(\X,\rho)$
to $[0,1]$.
Then the covering numbers of $\F_L$ may be estimated in terms of 
the covering numbers of $\X$:
\beq
\calN(\eps,\F_L,\nrm{\cdot}_\infty) \le \paren{\frac{3}{\eps}}^{ \calN(\eps/4L,\X,\rho)}.
\eeq
Hence, for doubling spaces with $\diam(\X)=1$,
\beq
\lognat \calN(\eps,\F_L,\nrm{\cdot}_\infty) 
\le 
\paren{\frac{8L}{\eps}}
^{\ddim(\X)}
\lognat
\paren{\frac{3}{\eps}}
.
\eeq
\end{lemma}
See Appendix for proof of the above.

\subsection{Rademacher complexities}
\label{sec:rad}
The (empirical)
{\em Rademacher complexity} 
\cite{DBLP:journals/jmlr/BartlettM02}
of 
a collection of functions $\F$
mapping some set $\calZ$ to $\R$
is 
defined,
with respect to a sequence
$Z=(Z_i)_{i\in[n]}\in \calZ^n$,
 by
\beqn
\label{eq:rad}
{\rad}_n(\F;Z) = 
\E
\sqprn{
\sup_{f\in\F} 
\oo{n}\sum_{i=1}^n \sigma_i f(Z_i)
},
\eeqn
where the expectation is over
the $\sigma_i$, which are 
iid
with $\P(\sigma_i=+1)=\P(\sigma_i=-1)=1/2$.

To any collection $\H$ of hypotheses 
mapping $\X$ to $ [\tau,1-\tau] $,
we associate
the {\em KL-loss class},
whose members
map $\X\times\R$ to $\R$.
The latter is denoted by
$\G$
and defined to be
\begin{align*}
\label{eq:loss-class}
& \G= \set{ g:(x,y)\mapsto -y \ln(h(x)) - (1-y)\ln(1-h(x))}
\end{align*}
where $h\in\H$.

We now move to prove the following variant of 
Theorem 3.3
in
\citet{mohri-book2012}:

\begin{theorem}
\label{thm:mohri-risk-bound}

For every $\delta>0$, we have that,
with probability at least $1-\delta$,
\beqn
\label{eq:rade-risk}
\risk(h) \le \risk_n(h) + 2\rad_n(\G;Z)
+3 \ln \truncb \sqrt{\frac{\lognat(2/\delta)}{2n}},
\eeqn
holds uniformly over all $h\in\H = [\trunc,1-\trunc]^\X $,
where $Z=(X_i,Y_i)_{i\in[n]}$ is the training sample.
\end{theorem}

\bepf
Let us Fix 2 samples $Z,Z'$ which differ by one point, and define the following function:
$$
\Phi(Z) = \sup_{g\in \G}{\paren{\E [g] - \hat {\E}_Z[g] }}
$$
Since the difference of suprema does not exceeds the suprema of the difference:
\begin{gather*}
\tsabs{\Phi(Z) - \Phi(Z')} \le \frac{1}{n} \sup_{g\in \mathcal{G}}\tsabs{g(z)-g(z')} = \\
\frac 1n \sup_{h\in \mathcal{H_L}}
\tsabs{ -y \ln{(h(x))} - (1-y) \ln{(1-h(x))} \\
+ y' \ln {(h(x'))} + (1-y') \ln{(1-h(x'))} }.
\end{gather*}

As $y,y'$ takes values in $\{0,1\}$, the above is equivalent to
$$ \frac 1n \sup_{h,h'\in [\trunc,1-\trunc] }
\left\{
    \ln \frac{h'}{h}
\right\}
\leq \frac1n \ln \frac{1-\trunc}{\trunc}
\leq \frac1n \ln \truncb.
$$

Due to symmetry we can bound $|\Phi(Z) - \Phi(Z')|\le \frac1n \ln {\truncb}$. Using McDiarmid's inequality we get that with probability of $1-\frac{\delta}{2}$:
$$ \phi(Z) \le 2\mathcal{R}_n(\G) + \ln {\frac 1 {\trunc}}\sqrt{\frac{\ln{\frac 2 {\delta}}}{2n}}, $$
Where $\E_Z[\phi(Z)] = 2\mathcal{R}_n(\G)$.
Using similar arguments we can bound $|\rad_n(\G;Z)-\rad_n(\G;Z')| \le \frac1n \ln \truncb$. Using McDiarmid a second time, we get that with probability of $1-\frac{\delta}{2}$:
$$  \mathcal{R}_n(\G) \le \rad_n(\G;Z) +  \ln {\frac 1 {\trunc}}\sqrt{\frac{\ln{\frac 2 {\delta}}}{2n}}.$$
Using union bound we get that for any $h \in \mathcal{H}_L, \delta > 0$, the following holds:
$$ \mathbb{E}(g) \le \frac1n\sum_{i=1}^n g(z_i) + 2\rad_n(\G;Z) + 3\ln \truncb \sqrt{\frac{\ln {\frac{2}{\delta}}}{2n}}. $$
\enpf

As $\H$ is a truncated away from zero, it is easy to notice that $\G$ is $\truncb$-Lipschitz. Using the following observation, we can apply Talagrand's contraction lemma to connect the Rademacher complexities, of $\H$ and $\G$.

Given a fixed sample $Z$ as defined above, for each pair $(X_i,Y_i)$ every $g \in \G$ is either $\log{h(X_i)}$ or $\log{1-h(X_i)}$, depending on $Y_i$. We get that both functions share the same Lipschitz constant of $\truncb$. Thus let us replace $g(z_i)$ by $\phi_i(x)$, where the latter replaces the corresponding logarithmic function.
Let us recall Talagrand's contraction
principle
\cite[Lemma 5.7]{mohri-book2012}:
\begin{lemma}
\label{lem:Talagrand}
Let $\phi_1,...,\phi_n$ be a $\truncb$-Lipschitz function class from $\R \mapsto{\R}$. Then, for any hypothesis set $\H$ of real-valued functions and any sample $Z=(X_i)_{i=1}^n$, the following holds:
\beqn
\label{eq:tal}
\E
\sqprn{
\sup_{h\in\H} 
\oo{n}\sum_{i=1}^n \sigma_i (\phi_i \circ h)(x_i)
}
\le 
\truncb{\rad}_n(\H;Z) .
\eeqn
\end{lemma}
Equipped with the covering numbers estimate, we proceed to bound the Rademacher complexity of Lipschitz functions on doubling spaces.
 
 \begin{theorem}
 \label{eq:rad_final}
 Let $\calF_L$ be a collection of $L$-Lipschitz [0,1]-valued functions defined on a metric space $(\X,\rho)$ with diameter 1 and doubling dimension $d\ge 1$. Then 
$$ {\rad}_n(\calF_L;Z) = O\paren{ \frac {L^{\frac d{d+1}}} {n^{\frac 1 {d+1}}}}. $$  
 \end{theorem}
See appendix for proof.

Combining Lemma~\ref{thm:mohri-risk-bound} with Lemma~\ref{lem:Talagrand} and Theorem\ref{eq:rad_final}, the generalization bound, Theorem ~\ref{thm:main-risk} is immediate.

\subsection{Choice of Truncation}{
\label{subsec:trunc}
The generalization bound
Theorem~\ref{thm:main-risk}
suggests an optimal truncation rate of
$\trunc = n^{-\frac 1{d+2}} $.
This is quite a bit more aggressive
than the truncation rates of
$\exp(-O(n))$
and
$\exp(-O(\sqrt n))$
necessitated by the lower bounds
in Section~\ref{sec:lb}.
Obtaining refined generalization bounds
that allow for less aggressive truncation
is an active research direction.

}

\newpage

\bibliography{refs}
\bibliographystyle{apalike}

\appendix


\onecolumn

\section{Deferred proofs}{
\subsection{Proof of Theorem~\ref{thm:func-gap} }{
\bepf
Let $w^*_t$ be the minimizer of the path-following scheme for a given $t$:
\beq w^*_t=\argmin_{w \in \dom F} f(w;t) \eeq
Thus:
$$ \nabla f(w^*_t;t)= t \nabla f_0(w^*_t) + \nabla F(w^*_t) = 0. $$
Additionally:
$$ | \langle u,v \rangle| \le \|u\|_{w,t}^*\|v\|_{w,t}. $$

We split the proof into 2 parts.
Observe that
\begin{align*}
 f_0(w_t)-f_0(w_{opt}) = 
 \underbrace{f_0(w_t)-f_0(w^*_t)}_{(i)}
+ \underbrace{f_0(w^*_t)- f_0(w_{opt})}_{(ii)}
\end{align*}

(i)  Since $f_0$ is self-concordant function, using \cite[Theorem 5.1.8]{nesterov} we get:
$$ f_0(w^*_t) \ge f_0(w_t)+ \langle \nabla f_0(w_t), w^*_t-w_t \rangle + \omega(\|w^*_t-w_t\|_{w_t}). $$
As $\omega(\tau) \ge 0 $  $\forall \tau \ge 0$, it can be dropped, leaving us with:
$$ f_0(w_t) -  f_0(w^*_t)\le -\langle \nabla f_0(w_t), w^*_t-w_t \rangle. $$

Multiplying by $t$ we get:
\begin{align*}
& t[f_0(w_t) -  f_0(w^*_t)]\le 
\langle -t\nabla f_0(w_t), w^*(t)-w_t \rangle \le 
 \|-t\nabla f_0(w_t)\|_{w_t,t}^* \|w^*_t-w_t\|_{w_t,t} \le \\
& \bigg[\|\nabla f(w_t;t)\|_{w_t,t}^* + \| \nabla F(w_t)\|_{w_t,0}^*\bigg] \|w^*_t-w_t\|_{w_t,t}\le  
 [\beta+\sqrt{v}] \frac{\beta}{1-\beta} .
\end{align*}
Where the last transition is due the definition of self-concordant barrier alongside lemma \ref{gno}, and \cite[Theorem 5.2.1]{nesterov} with the definition of the auxiliary function $\omega_*'$.

(ii) Similarly, we've:
$$  f_0(w^*_t) - f_0(w_{opt}) \le \langle - \nabla f_0(w^*_t), w_{opt}-w^*_t \rangle. $$

Multiplying by $t$ we get
\begin{align*}
&
t [f_0(w^*_t) - f_0(w_{opt}) ] \le \langle  -t \nabla f_0(w^*_t), w_{opt}-w^*_t \rangle = 
 \langle  \nabla F(w^*_t), w_{opt}-w^*_t \rangle \le v.
\end{align*}
The last transition is due to \cite[Theorem 5.3.7]{nesterov}.
\enpf
}

\subsection{Proof of Lemma~\ref{lem:inc_in_t}}{
\bepf
\begin{align*}
&t\|\nabla f_0(w)\|_{w,t}^* = 
[(t- \Delta t) + \Delta t] \|\nabla f_0(w)\|_{w,t}^* \le  
\|\nabla f(w;t-\Delta t)-\nabla F(w)\|_{w,t-\Delta t}^* + \Delta t\|\nabla f_0(w)\|_{w,t - \Delta t}^*  \\
& \le \beta + \sqrt{v} + \gamma.
\end{align*}
\enpf
}

\subsection{Proof of Theorem~\ref{thm:algo-analytic-comp}}
\bepf

Due to Lemma \ref{lem:inc_in_t}: 
\begin{align*}
&t_k = t_{k-1}\bigg(1+\frac{\gamma}{t_{k-1}\|\nabla f(w_{k-1})\|^*_{w_{k-1} , t_{k-1}}}\bigg) = \\
&t_1 \prod_{i=1}^{k-1} \bigg(1+\frac{\gamma}{t_i\|\nabla f(w_{i})\|^*_{w_{i} , t_{i}}} \bigg)  \ge \\
&\frac{\gamma}{\|\nabla f(w_0)\|^*_{w_0 , 0}} \bigg(1+\frac{\gamma}{\gamma+\beta+\sqrt{v}}\bigg)^{k-1} .
\end{align*}
 After some manipulations we get that $\forall k\ge1$:
\begin{align*}
& \ln {\paren{t_k \frac{\|\nabla f(w_0)\|^*_{w_0 , t_0}}{\gamma}}} \ge \paren{k-1}\frac{ \gamma}{\gamma+\beta+\sqrt{v}}\frac 12 \\
& 2\frac{\gamma+\beta+\sqrt{v}}{\gamma} \ln {\paren{t_k \frac{\|\nabla f(w_0)\|^*_{w_0 , t_0}}{\gamma}}} + 1 \ge k,
\end{align*}
where we used $\ln(1+x)\ge \frac x2$.
\enpf

\subsection{Proof of Lemma~\ref{lem:cov-cov}}
\bepf
Fix a covering of $\X$ consisting of 
$|N| = \calN(\eps', \X , \rho)$ balls $\{U_1,\dots,U_{|N|}\}$ of radius 
$\eps' = \eps/4L$ 
and choose $|N|$ points $N = \{x_i \in U_i\}_{i=1}^{|N|}$.
We will construct an $\eps$-cover $\wh F=\set{\hat f_1,\ldots,\hat f_{|\hat F|}}$ as follows.
At every point $x_i\in N$, we choose $\hat f(x_i)$ to be of the following form,
,while maintaining
$\tsLip{\hat f}\le 2L$:
$$ 2kL\eps',k=0,1,2,... $$

Construct a $2L$-Lipschitz  extension for $\hat f$ from $N$ to all over $\X$ 
(such an extension always exists, \cite{mcshane1934,Whitney1934}).
We claim that every $f \in \F_L $ is close to some $\hat f \in \wh F$,
in the sense that $\tsnrm{f-\hat f}_\infty \leq \eps$.
Indeed, every point $x\in \X$ is $\eps'$-close to some point $x_N\in N$,
and since $f$ is $L$-Lipschitz and $\hat f$ is $2L$-Lipschitz,
\begin{align*}
&\tsabs{f(x)-\hat f(x)} \leq \\
&\tsabs{f(x)-f(x_N)} + \tsabs{f(x_N)-\hat f(x_N)} + \tsabs{\hat f(x_N)-\hat f(x)} \leq \\
& L\cdot \rho(x,x_N) + L \eps' + 2L\cdot \rho(x,x_N) = 4L \eps' = \eps
\end{align*}

It is easy to verify that $\tsabs{\hat F} \leq (3/\eps)^{\tsabs{N}}$,
since by construction, the functions $\hat f$ are determined by their values on $N$.
This provides a covering of $\F_L$ using $\tsabs{\hat F}$ balls of radius $\eps$.

The bound for doubling spaces follows immediately by applying 
the so-called doubling property (see for example \cite{KL04}) 
and the diameter bound, to obtain
\beq
\calN(\eps',\X,\rho)&\le& 
\paren{\frac2{\eps'}}^{\ddim(\X)} = \paren{\frac {8L}{\eps}}^{\ddim(\X)}   
.
\eeq
\enpf

\subsection{Proof of Theorem~\ref{eq:rad_final}}{
\bepf
Recalling that
for norms induced by probability measures,
we have
$\|f\|_2 \le \|f\|_\infty$, we can substitute the estimate in Lemma~\ref{lem:cov-cov} to get:
\begin{align*}
& {\rad}_n(\calF_L;Z) \le 
 \inf_{\alpha \ge 0}{\paren{4\alpha  + 
12 \int_{\alpha}^{\infty}{\sqrt{\frac{\ln \calN (t,\calF,\infnrm)} n}dt}}} \le 
 \inf_{\alpha \ge 0}{\paren{4\alpha  + 
12 \int_{\alpha}^{\infty}{\sqrt{\frac{\paren{\frac{8L}t}
^d
\lognat
\paren{\frac{3}{t}}} n}dt}}} \le \\
& \inf_{\alpha \ge 0}{\paren{4\alpha  + 
12 \int_{\alpha}^{\infty}{\sqrt{\frac{\paren{\frac{8L}t}^d
\paren{\frac 3t} }n}dt}}} \le 
\inf_{\alpha \ge 0}{\paren{4\alpha  + 
\frac{24 (8L)^{d/2}}{\sqrt n} \int_{\alpha}^{\infty}{ t^{-\paren{\frac{d+1}2}} dt}}} = 
\inf_{\alpha \ge 0}{\paren{4\alpha  + 
K \alpha^{\frac{1-d}2} } } \\
& where 
\quad K = \frac{24 (8L)^{d/2}}{\sqrt n} \frac 2{d-1} %
\end{align*}
The optimal value is obtained by deriving:
$$ \alpha^* = \paren{\frac{K(d-1)}8}^\frac 2{d+1} $$
\\
Assuming $d\ge1$, by assigning $\alpha^*$ and taking $C$ to infinity we get:
$$
4\paren{\frac{K(d-1)} 8}^{\frac 2{d+1}} +
K \paren{\frac{K(d-1)} 8}^{\frac 2{d+1} \frac{1-d}2 }.
$$
Noticing that
$$
\paren{\frac{d-1} 8}^{{\frac{1-d}{1+d}}}, \quad
\paren{\frac{d-1} 8}^{{\frac 2{1+d}}} \le 3,
$$
We have
\begin{align*}
& 4\paren{\frac{K(d-1)} 8}^{\frac 2{d+1}} +
K \paren{\frac{K(d-1)} 8}^{\frac {1-d}{d+1}} \\
& \le 3 \paren{4K^{\frac 2{d+1}} + K^{\frac{2}{d+1}}} \le 15K^{\frac 2{d+1}}.
\end{align*}

Further assigning $K$ we get the stated bound with a constant $c=2520$.
\enpf
}

}
\section{Construction of the barrier}{
As all constraints are linear, the construction of the self-concordant Barrier is simple.
Though the box constraints are actually redundant (The objective can serve as a barrier for itself), since the Hessian of the Lipschitz constraint is singular, we've to maintain the box boundary to avoid it.

Additionally, the objective serves as the box barrier (except for the case that all samples are 1 or 0, which is trivial), so only the Lipschitz constraints are considered.\\
Let $ \LM $ be a symmetric $n\times n$ matrix, where $\LM_{ij}= L d_T(X_i,X_j)$.\\
Let $ \zeta(i,j) = (w_i-w_j+\LM_{ij})$.
$$
F(w) = F_Q(w) + F_L(w)\quad, 
F_Q(w) = - \sum_{i=1}^n{[\ln{(w_i)} + \ln{(1-w_i)}]},\quad
F_L(w) = - \sum_{1\le i < j \le n}{[\ln{\zeta(j,i)} + \ln{\zeta(i,j)}]}.
$$
The gradient and hessian are:
\begin{equation*}
\begin{aligned}[t]
& \frac{\partial F_Q(w)}{\partial w_i} = -w_i^{-1}+\paren{1-w_i}^{-1}, \\
& \frac{\partial^2 F_Q(w)}{\partial^2 w_i} = w_i^{-2}+\paren{1-w_i}^{-2} , \\
& \frac{\partial^2 F_Q(w)}{\partial w_i \partial w_j} = 0 
\end{aligned}
\qquad
\begin{aligned}[t]
& \frac{\partial F_L(w)}{\partial w_i} = \sum_{j\ne i}{[\zeta(j,i)^{-1}-\zeta(i,j)^{-1}]} \\
& \frac{\partial^2 F_L(w)}{\partial w_i \partial w_j} = -\zeta(j,i)^{-2}-\zeta(i,j)^{-2} \\
& \frac{\partial^2 F_L(w)}{\partial^2 w_i} = \sum_{j\ne i}{[\zeta(j,i)^{-2}+\zeta(i,j)^{-2}]} =
- \sum_{j\ne i}{\frac{\partial^2 F_L(w)}{\partial w_i \partial w_j}}.
\end{aligned}
\end{equation*}

The barrier parameter is dominated by the Lipschitz constraints, thus we can set $$v= (n-1)n = O(n^2).$$
In order to calculate $ \| \nabla f_0(w_0)\|^*_{w_0 , 0}$, first note that since $t=0$:
$$ \nabla^2 f(w_0) = \nabla^2 F(w_0) \succeq \nabla^2 f_0(w_0) .$$
 Since we choose $w_0$ s.t. $w_0(i) = 1/2, i\in[n] $, then $\nabla ^2 F_L(w_0) = 0$.
Thus, it can be calculated directly:
\begin{align*}
&\|\nabla f_0(w_0)\|^*_{w_0 , 0} \le  
{\langle [\nabla^2 f_0(w_0)]^{-1} \nabla f_0(w_0) , \nabla f_0(w_0) \rangle}^{1/2}  = 
\sqrt{\sum_{i=1}^n {\paren{ \paren{\frac 12}^{-2}}^{-1}\paren{\paren{-\frac 12}^{-1}}^2} } = 
\sqrt{\sum_{i=1}^n {\frac 14 \cdot 4}} = \sqrt{n} .
\end{align*}
}

\section{Anti-concentration lemmas}

The following ``Reverse Chernoff bound'' is due to \citet{14476}:
\begin{lemma}
\label{lem:rev-chern}
Suppose that $X\sim\Binom(n,p)$,
and $0<\eps,p\le1/2$
satisfy $\eps^2pn\ge3$.
Then
\beq
\P(X\le (1-\eps)pn)
&\ge&
\exp\big({-9\eps^2 pn}\big),\\
\P(X\ge (1+\eps)pn)
&\ge&
\exp\big({-9\eps^2 pn}\big).
\eeq
\end{lemma}

\begin{corollary}
\label{cor:binom-diff}
There is a universal constant
$c>0$ for which the following holds.
Let $\mu$ be the uniform
distribution on the $6$-point
set $\X=\set{1,\ldots,6}$
and let $\mu_n$ be its empirical
realization induced by an iid sample
of size $n$.
Define the random variables
$X=n\mu_n(\set{1})$
and
$X'=n\mu_n(\set{2})$;
each is 
distributed
according to
$\Binom(n,p)$,
where $p=1/6$
($X,X'$ are not independent).
For $n\ge324$,
we have
\beq
\P(X-X'> 2\sqrt{n})
&\ge&
c.
\eeq
\end{corollary}
\bepf
We choose 
$\eps
=
1.5/(p\sqrt n)\le
1/2$
and verify that $\eps^2pn=2.25/p=13.5>3$,
so the conditions of Lemma~\ref{lem:rev-chern} hold.
Now $\eps n p=1.5\sqrt n$
and we have that
$\P(X\ge np+1.5\sqrt{n})\ge\exp(-20.25/p)=e^{-121.5}$
and
$\P(X'\le np-1.5\sqrt{n})\ge e^{-121.5}$.
By inclusion of events,
\beq
\P(X-X'> 2\sqrt{n})
&\ge&
\P(X\ge np+1.5\sqrt{n},X'\le np-1.5\sqrt{n})
.
\eeq
Observe that
$X,X'$ are {\em negatively associated}
\cite{Dubhashi:1998:BBS:299633.299634},
which implies that
$\P(X\ge s,X'\le t)\ge
\P(X\ge s)
\P(X'\le t)
$
for all $s,t$. It follows that
\beq
\P(X-X'> 2\sqrt{n})
&\ge&
\P(X\ge np+1.5\sqrt{n})\cdot\P(X'\le np-1.5\sqrt{n})
\ge
e^{-243}
=:c
>0.
\eeq
\enpf

\end{document}